\documentclass[letterpaper, 10 pt, conference]{ieeeconf}  

\IEEEoverridecommandlockouts                              

\usepackage{xcolor}
\usepackage{amssymb}
\usepackage{graphicx}
\usepackage{mathtools}
\usepackage{amsmath}
\usepackage{gensymb}
\usepackage{float}
\usepackage{multirow}
\usepackage{commath}
\usepackage{makecell}
\usepackage{tabularx}
    \newcolumntype{L}{>{\raggedright\arraybackslash}X}
\usepackage[utf8]{inputenc}
\usepackage[english]{babel}

\usepackage{amsthm}

\usepackage{subcaption}

\newcommand{\no}{\noindent}

\newtheorem{prop}{Proposition}[section]

\title{\LARGE \bf
DWA-RL: Dynamically Feasible Deep Reinforcement Learning Policy for Robot Navigation among Mobile Obstacles 
}

\author{Utsav Patel, Nithish K Sanjeev Kumar, Adarsh Jagan Sathyamoorthy and Dinesh Manocha.
\thanks{ This work was supported in part by ARO Grants W911NF1910069, W911NF1910315 and Intel.} 
}

\begin{document}

\maketitle
\thispagestyle{empty}
\pagestyle{empty}

\begin{abstract}

We present a novel Deep Reinforcement Learning (DRL) based policy to compute dynamically feasible and spatially aware velocities for a robot navigating among mobile obstacles. Our approach combines the benefits of the Dynamic Window Approach (DWA) in terms of satisfying the robot's dynamics constraints with state-of-the-art DRL-based navigation methods that can handle moving obstacles and pedestrians well. Our formulation achieves these goals by embedding the environmental obstacles' motions in a novel low-dimensional observation space. It also uses a novel reward function to positively reinforce velocities that move the robot away from the obstacle's heading direction leading to significantly lower number of collisions. We evaluate our method in realistic 3-D simulated environments and on a real differential drive robot in challenging dense indoor scenarios with several walking pedestrians. We compare our method with state-of-the-art collision avoidance methods and observe significant improvements in terms of success rate (up to 33\% increase), number of dynamics constraint violations (up to 61\% decrease), and smoothness. We also conduct ablation studies to highlight the advantages of our observation space formulation, and reward structure.
\end{abstract}

\section{Introduction}

There has been considerable interest in using Deep Reinforcement Learning (DRL)-based local planners \cite{crowdsteer,densecavoid,everett2019collision,Multirobot_collison_avoidance} to navigate a non-holonomic/differential drive robot through environments with moving obstacles and pedestrians. They are effective in capturing and reacting to the obstacles' motion over time, resulting in excellent mobile obstacle avoidance capabilities. In addition, these methods employ inexpensive perception sensors such as RGB-D cameras or simple 2-D lidars and do not require accurate sensing of the obstacles. However, it is not guaranteed that the instantaneous robot velocities computed by DRL-based methods will be \textit{dynamically feasible} \cite{non-holonomic-constraints,DWA}. That is, the computed velocities may not obey the acceleration and non-holonomic constraints of the robot, becoming impossible for the robot to move using them. This leads to highly non-smooth and jerky trajectories. 

Desirable behaviors such as computing dynamically feasible velocities are developed using a DRL method's reward function, where they are positively rewarded and undesirable behaviors such as collisions are penalized. However, a fully trained policy could over-prioritize the collision avoidance behavior over dynamic feasibility, if the penalty for collision is not appropriately balanced with the reward for computing feasible velocities \cite{rewardBalancing}. Therefore, acceleration limits and the non-holonomic constraints of the robot may not be satisfied. It is crucial that the policy account for such fundamental constraints especially when the robot navigates among pedestrians and other mobile obstacles. 

Another issue with such methods \cite{crowdsteer,densecavoid} is that they use high-dimensional data such as RGB or depth images as inputs during training to detect and observe obstacles. This greatly increases the overall training time and makes it harder for the policy to generalize the behaviors learnt in one environment to another. 

\begin{figure}[t]
      \centering
      \includegraphics[width=\columnwidth,height=5.5cm]{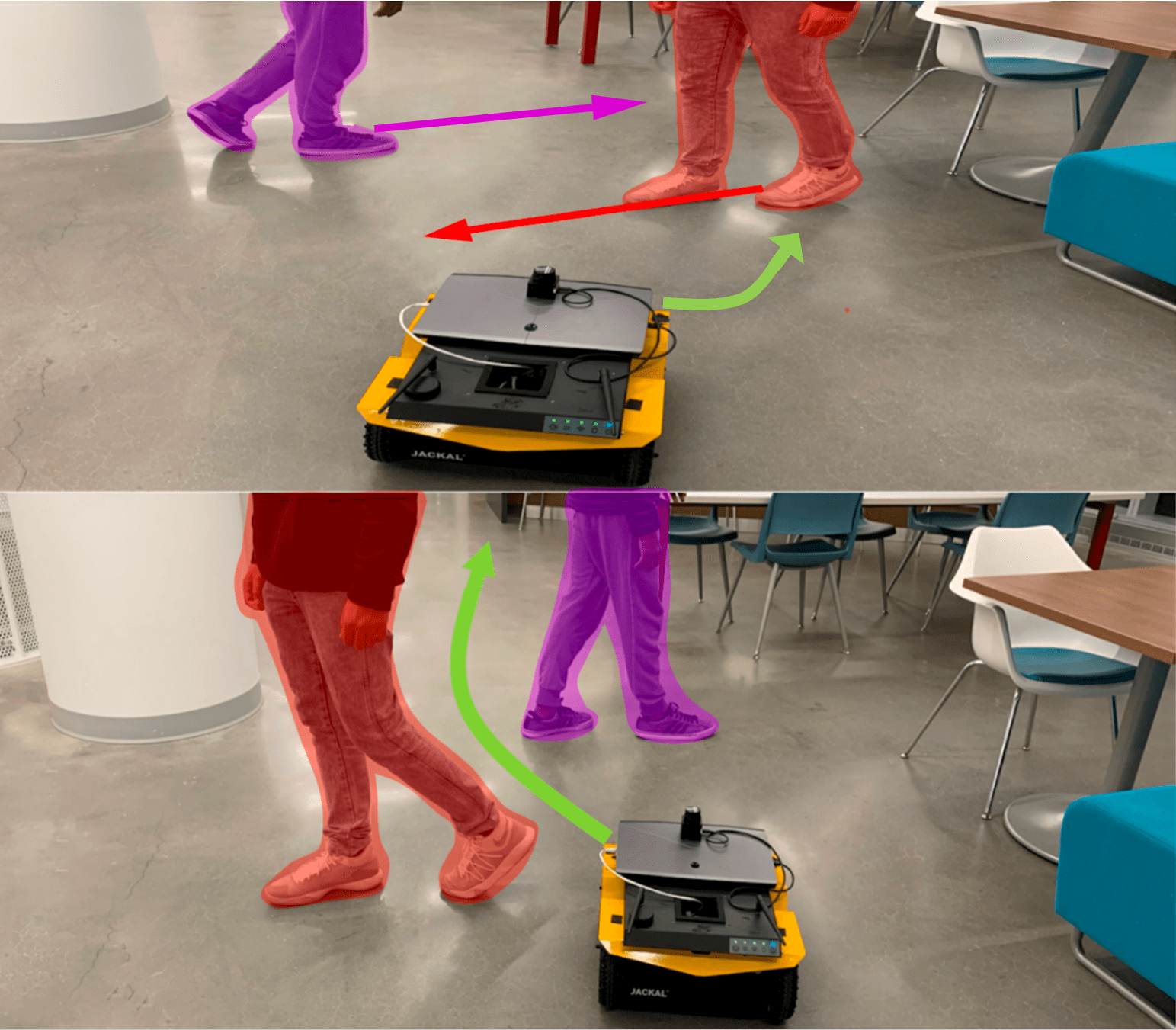}
      \caption {\small{Our robot avoiding mobile obstacles using dynamically feasible, smooth, spatially-aware velocities. The red and violet arrows indicate the obstacles' motion and green arrows shows the robot's trajectory in two time instances for different obstacle positions. Our hybrid approach, DWA-RL, considers the motion of the moving obstacles over time in its low-dimensional observation space which is used to compute the robot velocities. This results in fewer collisions than DWA \cite{DWA}}, and DRL-based methods \cite{Multirobot_collison_avoidance}. Since our method computes the robot velocities based on DWA's feasible velocity space, the computed robot velocities are guaranteed to obey the acceleration and non-holonomic constraints of the robot.}
      \label{fig:cover-image}
      \vspace{-10pt}
\end{figure}

On the other hand, the Dynamic Window Approach (DWA) \cite{DWA}, is a classic navigation algorithm that accounts for the robot's dynamics constraints and guarantees that the velocities in a space known as the \textit{dynamic window} are collision-free and feasible/achievable for the robot within a time horizon $\Delta t$. However, DWA's formulation only considers the robot's sensor data at the current time instant to make decisions. As a result, avoiding mobile obstacles becomes challenging, leading to higher number of collisions \cite{PredictiveDWA}.

{\bf Main Results:}  We present a hybrid approach, DWA-RL, that combines the benefits of DWA and DRL-based methods for navigation in the presence of mobile obstacles. We present a DRL-based collision avoidance policy that utilizes a novel observation space formulation and a novel reward function to generate spatially aware, collision-free, dynamically feasible velocities for navigation. We show that our approach has a superior performance compared to DWA and a DRL-based method \cite{Multirobot_collison_avoidance} in terms of success rate, number of dynamics constraints violations, and smoothness. The main contributions of our work include:

\begin{itemize}
\item A novel formulation for the observation space, based on the concept of dynamic window, is used to train our DRL-based navigation policy. The observation space is constructed by calculating the robot's feasible velocity set at a time instant and the costs corresponding to using those velocities in the past \textit{n} time instants. This formulation embeds the time evolution of the environment's state and preserves the dynamic feasibility guarantees of DWA (Section \ref{Sec:Our_Approach}). This leads to a significantly lower dimensional observation space unlike other DRL methods \cite{crowdsteer,densecavoid}. This also results in significantly lower training times, and easier sim-to-real transfer of the fully trained policy.

\item A novel reward function that is shaped such that the robot's navigation is more spatially aware of the obstacles' motion. That is, the robot is rewarded for navigating in the direction opposite to the heading direction of obstacles. This leads to the robot taking maneuvers around moving obstacles. This is different from DWA, which might navigate directly into the path of a mobile obstacle or collide with it. Overall, our approach reduces the collision rate by 33\% in dynamic environments as compared to DWA.

\end{itemize}

We evaluate our method and highlight its benefits over prior methods in four high-fidelity 3-D simulated environments that correspond to indoor and outdoor scenes with many static and moving obstacles. To demonstrate the sim-to-real capabilities of our method, we use DWA-RL to navigate a real differential drive robot using a simple 2-D lidar in indoor scenes with randomly walking pedestrians.
\section{Related Work}

\subsection{Collision Avoidance in Dynamic Scenes}
Global collision avoidance methods \cite{dijkstra, Astar, RRT} compute an optimal trajectory for the entire route, but they generally work offline which is not suitable for dynamic obstacles. On the other hand, vector-based local approaches such as DWA\cite{DWA} or other multi-agent methodss~\cite{wolinski2014parameter} use limited sensory information and are computationally efficient when avoiding static obstacles.

Several works have extended DWA's capabilities to avoid mobile obstacles by using techniques such as D* graph search \cite{D*DWA}, look-ahead to detect non-convex obstacles \cite{vspaceApproach}, or by extending beyond the local dynamic window to compute possible future paths using a tree \cite{DW4DOT&DW4DO}. The Curvature-Velocity method \cite{curvatureVM} is another method similar to DWA which formulates collision avoidance as a constrained optimization problem incorporating goal and vehicle dynamics.

\subsection{DRL-based Collision Avoidance}
There have been numerous works on DRL-based collision avoidance in recent years. Methods such as \cite{monocularVisionRL} use a deep double-Q network for depth prediction from RGB images for collision avoidance in static environments, whereas more advanced methods \cite{CNNvisuomotorPolicy} use Convolutional Neural Networks to model end-to-end visuomotor navigation capabilities. 

An end-to-end obstacle avoidance policy for previously unseen scenarios filled with static obstacles a few pedestrians is demonstrated in \cite{DLnavigation}. A decentralized, scalable, sensor-level collision avoidance method was proposed in \cite{Multirobot_collison_avoidance}, whose performance was improved using a new hybrid architecture between DRL and Proportional-Integral-Derivative (PID) control in \cite{JiaPan2}. Assuming that pedestrians aid in collision avoidance, a cooperative model between a robot and pedestrians was proposed in \cite{JHow1} for sparse crowds. An extension to this work using LSTMs \cite{LSTMobstacletrackingRL} to capture temporal information enabled it to operate among a larger number of pedestrians.

A few deep learning-based works have also focused on training policies that make the robot behave in a socially acceptable manner \cite{gail,socially-aware} and mitigate the freezing robot problem \cite{densecavoid,Frozone}. However, such policies do not provide any guarantees on generating dynamically feasible robot velocities.


\section{Background}
In this section we provide an overview of the different  concepts and components used in our work.

\subsection{Symbols and Notations}
A list of symbols frequently used in this work is shown in Table \ref{tab:symbol_defn}. Rarely used symbols are defined where they are used.
\begin{table}[]
    \centering
    \begin{tabularx}{\linewidth}{|c|L|} 
\hline
\textbf{Symbols} & \textbf{Definitions}  \\
\hline
k & Number of linear or angular velocities robot can execute at a time instant \\
\hline
n & Number of past time instants used in observation space formulation \\ 
\hline
$v, \omega$ & Robot's linear and angular velocities \\
\hline
$\dot{v}_l, \dot{\omega}_l$ & Robot's linear and angular acceleration limits \\
\hline
\textit{distobs}$^{t}(v, \omega$) & Function gives the distance to the nearest obstacle at time \textit{t} from the trajectory generated by velocity vector($v, \omega$) \\
\hline
\textit{dist}($\textbf{p}_1, \textbf{p}_2)$ & Euclidean distance between $\textbf{p}_1$ and $\textbf{p}_2$ (two 2-D vectors). \\
\hline
$v_{a}, \omega_a$ & Robot's current linear and angular velocity \\ 
\hline
$d_t$ & Distance between the pedestrian and robot \\
\hline
r & Reward for a specific robot action during training \\ 
\hline
$R^{rob}$ & Robot's radius \\
\hline
$\textbf{p}^t_{rob}$ & Position vector of the robot in the odometry frame at any time instant t \\
\hline
$\textbf{g}$ & Goal location vector \\
\hline
$EndPoint(v_i, \omega_i)$ & Function to compute the end point of an trajectory generated by a ($v_i, \omega_i$) vector \\

\hline
\end{tabularx}
    \caption{\small{List of symbols used in our work and their definitions.}}
    \label{tab:symbol_defn}
\vspace{-15pt}
\end{table}

\subsection{Dynamic Window Approach}
The Dynamic Window Approach (DWA) \cite{DWA} mainly uses the following two stages to search for a collision-free, and reachable [v, $\omega$] velocity vector in a 2-dimensional velocity space known as the dynamic window. The dynamic window is a discrete space with $k^2$ [v, $\omega$] velocity vectors, where k is the number of linear and angular velocities that the robot can execute at any time instant.

\subsubsection{Search Space}
The goal of the first stage is to generate a space of reachable velocities for the robot. This stage involves the following steps.

\textbf{Velocity Vectors Generation:}
In this step, according to the maximum linear and angular velocities the robot can attain, a set V of [v, $\omega$] vectors is generated. Each velocity vector in the set corresponds to an arc of a different radius along which the robot can move along. The equations describing the trajectory of the robot for different [v, $\omega$] vectors can be found in \cite{DWA}. 

\textbf{Admissible Velocities:}
After forming set V, for each $[v, \omega] \in V$, the distance to the nearest obstacle from its corresponding arc is computed. The [v, $\omega$] vector is considered \textit{admissible} only if the robot is able to stop before it collides with the obstacle. The admissible velocity set $V_{ad}$ is given by, 

\begin{equation}
\begin{split}
V_{ad} = \{
 v, \omega \} \hspace{0.75cm} \text{Where, } 
 v \leq \sqrt{2 \cdot distobs(v, \omega) \cdot \dot{v_b}}, \\
 \omega \leq \sqrt{2 \cdot distobs(v, \omega) \cdot \dot{{\omega}_b}} 
 \end{split}
\end{equation}

\textit{dist(v, $\omega$)}, is the distance to the nearest obstacle on the arc.

\textbf{Dynamic Window:}
The next step is to further prune the set $V_{ad}$ to remove the velocities that are not achievable within a $\Delta t$ considering the robot's linear and angular acceleration limits. This final set is called the dynamic window and is formulated as,

\begin{multline} \label{DynamicFeasibility}
    V_d = \{v, \omega | v \in [v_a - \Dot{v}_l \cdot \Delta t, v_a + \Dot{v}_l \cdot \Delta t ], \\
\omega \in [{\omega}_a - \Dot{\omega}_l \cdot \Delta t, {\omega}_a + \Dot{\omega}_l \cdot \Delta t ] \}.
\end{multline}

\subsubsection{Objective Function Optimization}
In the second stage, the $[v, \omega]$, which maximizes the objective function defined in equation \ref{DWA-objective-function}, is searched for in $V_d$.

\begin{equation}
G (v, \omega) =  \sigma( \alpha.heading(v, \omega) + \beta.distobs(v, \omega) + \gamma.vel(v,\omega)).
\label{DWA-objective-function}
\end{equation}

For a [v, $\omega$] executed by the robot, \textit{heading()} measures the robot's progress towards the goal (more progress $\implies$ higher value), \textit{dist()} measures the robot's distance from the nearest obstacles (more distance $\implies$ higher value), and the \textit{vel()} function checks that $v \ne 0$. $\alpha, \beta$ and $\gamma$ denote weighing constants that can be tuned by the user.

Obstacle information embedded in the velocity space is utilized to select the optimal velocity pair. The [v, $\omega$] vector computed by DWA may be a local minimum. However, this issue can be mitigated if the connectivity of free space to the goal is known.


\begin{figure*}[t]
      \centering
      \includegraphics[width=\textwidth, height=6cm]{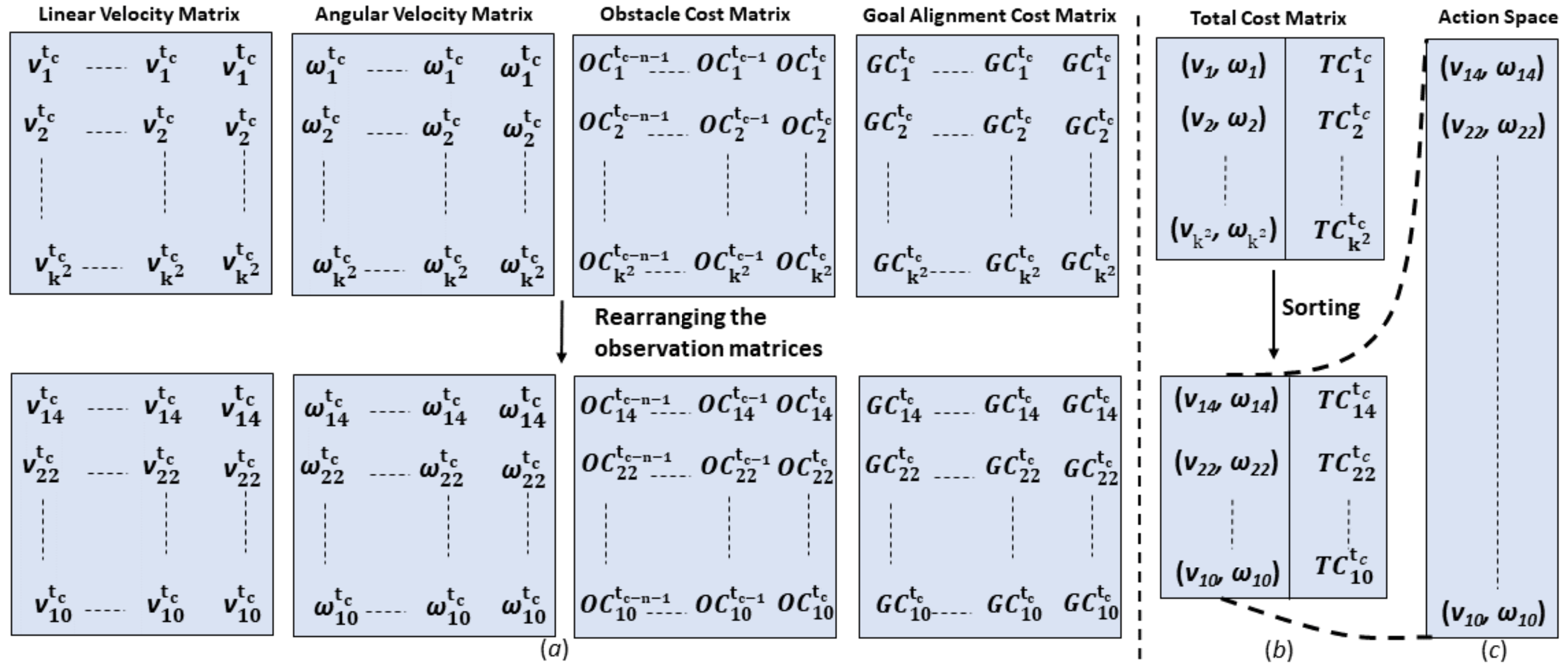}
      \caption{\small{\textbf{(a)[Top]} The initial construction of our observation space. Initially, the linear and angular velocity matrices ( $[v, \omega] \in V_f$) along with their obstacle and goal alignment cost matrices for \textit{n} time instants are constructed. \textbf{[Bottom]} The values in the four matrices are rearranged in the descending order of the total cost value at the current time instant. In this case shown, [$v_{14}, \omega_{14}$] has the least total cost. \textbf{(b)[Top]} The total costs TC for each velocity vector belonging to set $V_f$. \textbf{[Bottom]} The velocity vectors rearranged in the descending order of the total cost at the current time instant. \textbf{(c)} The action space for the current time step is obtained by sorting the feasible velocities vectors $(v, \omega)$ in the descending order of the total cost value at the current time instant. }}
      \label{ObservationSpaceGeneration}
      \vspace{-15pt}
\end{figure*}

\subsection{DRL Policy Training} \label{Sec:PPO}
DRL-based collision avoidance policies are usually trained in simulated environments (similar to Fig.\ref{fig:Testing-Scenarios}) using a robot that uses the said policy to perform certain \textit{actions} based on environmental \textit{observations} to earn some \textit{rewards}. The robot's observation consists of information regarding its environment (such as the positions of obstacles), and the set of all observations that the robot's sensors can make is called its \textit{observation space} ($\textbf{o}^t$). The robot's actions are represented by the velocities that it can execute, and the set of all the robot's velocities is called its \textit{action space} ($\textbf{a}^t$).

The policy's objective during training is to maximize a \textit{reward function} by performing the actions which are rewarded and avoiding actions that are penalized. This proceeds until the robot continuously achieves the maximum reward for several consequent training iterations. Collision-free velocities can then be computed from the fully trained policy $\pi$ as,

\vspace{-15pt}
\begin{equation}
    [v, \omega] \sim \pi(\textbf{\textbf{a}}^t | \textbf{o}^t).
\end{equation} 

\section{Our Approach} \label{Sec:Our_Approach}
In this section, we explain the construction of our novel observation space, the reward function, and our network architecture. 
 \vspace{-5pt}
\subsection{Observation Space Generation} \label{Sec: Obs_Space}
The steps used in the observation space construction are detailed below. 

\subsubsection{Dynamically Feasible Velocity Vectors}
Unlike DWA, we do not first generate an admissible velocity set that contains collision-free robot velocities. Instead, we first compute sets of feasible/reachable linear and angular velocities (\textit{lin} = $[v_a - \Dot{v} \cdot \Delta t, v_a + \Dot{v} \cdot \Delta t]$ and \textit{ang} = $[{\omega}_a - \Dot{\omega} \cdot \Delta t, {\omega}_a + \Dot{\omega} \cdot \Delta t ]$) using equation \ref{DynamicFeasibility}. We discretize these sets \textit{lin} and \textit{ang} into k intervals such that the total number of [v, $\omega$] vectors obtained from the intervals is $k^2$. We then form the set of feasible velocities $V_f$ from these discretized sets as, 

\vspace{-10pt}
\begin{equation}
    V_f = \{(v, \omega) | v \in lin_k, \omega \in ang_k\}.
\end{equation}
\vspace{-15pt}

The velocity vectors in $V_f$ do not account for the locations of the obstacles in the current time instant $t_c$ or the past \textit{n}-1 time instants. Therefore, some velocities in $V_f$ could lead to collisions. The \textit{k} linear and angular velocities in $V_f$ are appended \textit{n}-1 times as column vectors in two matrices each of size ($k^2 \times n$) and the generated linear and angular velocity matrices are shown in the Fig. \ref{ObservationSpaceGeneration}(a).
\begin{figure}[h!]
      \centering
      \includegraphics[width=\columnwidth,height=4.0cm]{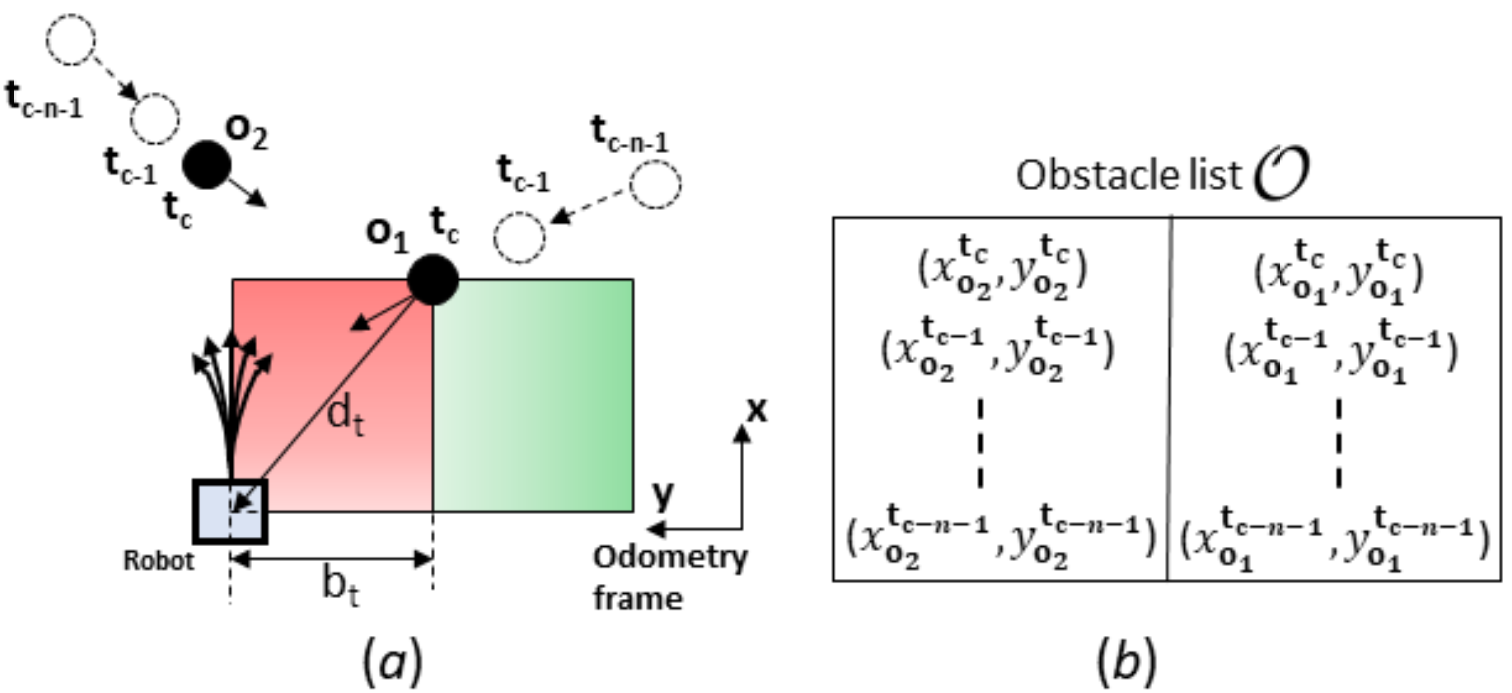}
      \caption { \small{Obstacle set construction. \textbf{(a)} The change in the location of the obstacle in the past n-1 time instances with respect to the location of the robot in the current time instance. The red region ($\mathcal{R}$) and the green region ($\mathcal{G}$) denote the regions of high risk and low risk of collisions respectively. They depend on the relative positions and motions of the robot and the obstacle. \textbf{(b)} The list of the obstacle sets obtained at various time instances each column corresponds to the location of the obstacles at particular time instance}}
      \label{fig:ObstacleLocations}
\end{figure}

\subsubsection{Obstacle sets}
We use a 2-D lidar scan to sense the location of the obstacles around the robot. For each time instant, the obstacle locations are obtained relative to a fixed odometry coordinate frame and stored in a set. The odometry frame is attached to the ground at the location from where the robot started. In Fig. \ref{fig:ObstacleLocations}(a), the locations of two obstacles in the current as well as in the past \textit{n}-1 time steps are shown. We add the set of obstacle locations in a list $\mathcal{O}$ of length \textit{n} (see Fig. \ref{fig:ObstacleLocations}(b)), where each row shows the set of obstacle locations for a specific time instant. We use $\mathcal{O}$ to incorporate information regarding the motion of various obstacles in the environment.

\subsubsection{Obstacle cost calculation}
Next, we calculate the \textit{obstacle cost} for every velocity vector in $V_f$ using the $distobs^t()$ function. Each vector in $V_f$ is forward simulated for a time duration $\Delta t$ to check if it leads to a  collision, given the obstacle positions in $\mathcal{O}$. The costs are calculated as,

\vspace{-10pt}
\begin{equation}
    OC^{t_j}_i = 
    \begin{cases}
    \,c_{col}    \hspace{2cm}  \text{if} \, distobs^{t_j}(v_i, \omega_i) < R^{rob}, 
      \qquad \qquad \\
     \cfrac{1}{distobs^{t_j}(v_i, \omega_i)}   \hspace{0.75cm} \text{otherwise}. \\
    \end{cases}
\end{equation}

\no Where, $c_{col}$= 40. The Fig.\ref{ObservationSpaceGeneration} (a) shows the obtained ($k^2 \times n$) obstacle cost matrix.

\subsubsection{Goal alignment cost calculation} 
Each $[v, \omega]$ in $V_f$ is forward simulated for a time $\Delta t$ and the distance from the endpoint of the trajectory to the robot's goal is measured (equation \ref{eqn:GC1}). The velocity vectors that reduce the distance between the robot and the goal location are given a low cost. 

\begin{equation}
    GC_i^{t_c} = dist(EndPoint(v_i, \omega_i), \textbf{g}) \hspace{0.1cm} * \hspace{0.10cm} c_{ga}
    \label{eqn:GC1}
\end{equation}

\begin{figure*}[h!]
      \centering
      \includegraphics[width=\textwidth,height=4.5cm]{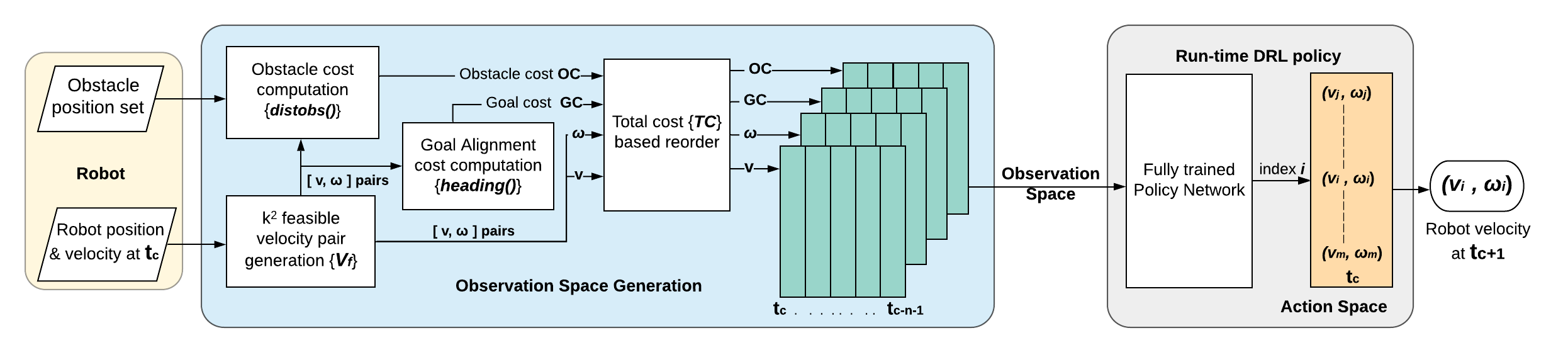}
     \caption{\small{Our method's run-time architecture. The observations such as obstacle positions measured by the robot's sensors (lidar in our case) and the robot's position and velocity at time $t_c$, along with the obstacle and goal-alignment costs are reordered (Section \ref{Sec:Pre-process}) to generate a ($k^2 \times n \times 4$) dimensional observation space (Section \ref{Sec:Obs_Space}) shown in green corresponding to time instant $t_c$. The fully trained DRL policy network (shown in Fig. \ref{fig:Network Architecture}) uses the observation space to compute the index of the output velocity in the action space.}}
     \label{flow_chart}
     \vspace{-10pt}
 \end{figure*}

The goal alignment cost is independent of the location of the obstacles around the robot, therefore the same cost for each pair is appended \textit{n} times to obtain a goal alignment cost matrix of shape ($k^2 \times n$) as seen in Fig. \ref{ObservationSpaceGeneration}(a), and in the equation \ref{eqn:GC2}.

\vspace{-10pt}
\begin{equation}
    GC_i^{t_c} = GC_i^{t_{c-1}} =......= GC_i^{t_{c-n-1}}
    \label{eqn:GC2}
\end{equation}

\no Where, $c_{ga}$= 2.5.

\subsubsection{Total cost calculation}
The total cost for the robot using a vector [$v_i, \omega_i$] for the current time instant $t_c$ is calculated as,

\vspace{-10pt}
\begin{equation}
    TC_i^{t_c} = OC^{t_c}_i + GC^{t_c}_i
\end{equation}

\no and is shown in Fig.\ref{ObservationSpaceGeneration}(b).

\subsubsection{Sorting the Velocity Vectors}\label{Sec:Pre-process}
The linear, angular, obstacle cost and goal alignment cost matrices obtained in Section \ref{Sec: Obs_Space} are now reordered to better represent which velocities in $V_f$ have the lowest costs given the obstacle positions for the past \textit{n} time instants. The velocity vectors are sorted in ascending order according to the total cost of the velocity vectors at the current time instant. The elements in the velocity and cost matrices are then reordered in same order.

\subsubsection{Observation Space and Action Space} \label{Sec:Obs_Space}
Finally, our observation space is constructed using the reordered linear, angular matrices along with the obstacle and goal alignment cost matrices and stacking them to get a matrix of size ($k^2 \times n \times 4$). Our action space is the reordered set of feasible velocities for the robot at the current time instant (see Fig.\ref{ObservationSpaceGeneration}c). The observation space is then passed to the policy network (see Fig.\ref{flow_chart}).

\subsection{DRL Navigation Framework}
In this section, we detail the other components of our DRL policy's training, and run-time architecture.

\subsubsection{Reward Function Shaping}
Rewards for the basic navigation task of reaching the goal and avoiding collisions with obstacles are provided with high positive or negative values respectively. In order to make the training faster, the difference between distance from the goal in the previous and the current time instant is utilized in the reward function. This incentivizes the policy to move the robot closer to the goal each time step, or otherwise be penalized as, 

\vspace{-10pt}
\begin{equation}
    (r_g)^t =
    \begin{cases}
     r_{goal} \qquad \qquad \qquad \qquad \text{if} \, dist(\textbf{p}^{t}_{rob}, \textbf{g}) < 0.3 \text{m},\\
     -2.5(dist(dist(\textbf{p}^t_{rob}, \textbf{g})) - \textbf{p}^{t-1}_{rob}, \textbf{g}) \qquad   \text{otherwise}.
    \end{cases}
\end{equation}

\no We define the penalty for collision as,
\vspace{-10pt}
\begin{equation}
    (r_c)^t =
    \begin{cases}
     r_{collision} \ \qquad \qquad \text{if} \, dist(\textbf{p}^t_{rob}, \textbf{p}^t_{obs}) < 0.5 \text{m},\\
     0  \hspace{2.6cm}    \text{otherwise}.
    \end{cases}
\end{equation} 

\begin{figure}[t]
      \centering
      \includegraphics[height=3.0cm,width=\columnwidth]{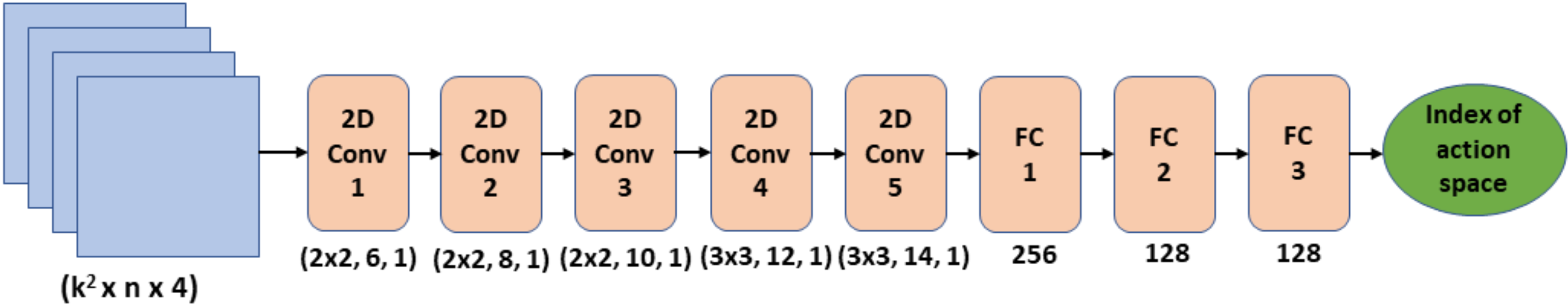}
      \caption {\small{Architecture of the policy network used for training. The input observation space is marked in blue, the network layers are marked in orange. The initial 5 layers are the convolutional layers and the remaining 3 layers are the fully connected layers. The output of the network is marked in green. }}
      \label{fig:Network Architecture}
\end{figure}


When the distance between a dynamic obstacle (located at $\textbf{p}^t_{obs} = [x^t_{obs}, y^t_{obs}]$) and the robot (located at $\textbf{p}^t_{rob} = [x^t_{rob}, y^t_{rob}]$) is less than a certain threshold, the robot receives the steering reward/penalty (equation \ref{eqn:spatial-awareness}). The parameters $d_t$ and $b_t$ which influence this reward are depicted in Fig.\ref{fig:ObstacleLocations}a, and defined as follows,

\begin{equation}
d_t = dist(\textbf{p}^t_{rob}, \textbf{p}^t_{obs}) \hspace{0.5cm} | \hspace{0.5cm} b_t =  y^t_{rob} - y^t_{obs}. 
\end{equation}

\vspace{-15pt}
\begin{equation}
    (r_{steer})^{t} =
    \begin{cases}
    -  | b_t | * r_{spatial} -
     \cfrac{r_{proximity}}{d_t}  \,\hspace{.3cm}  \text{if} \,  \hspace{.1cm} {p}^t_{rob} \in \mathcal{R}\\
    + | b_t | *  r_{spatial}   \hspace{2.3cm} \text{if} \hspace{.1cm}{p}^t_{rob} \in \mathcal{G}.\\
     
    \end{cases}
\label{eqn:spatial-awareness}
\end{equation} 

From equation \ref{eqn:spatial-awareness}, it can be seen that the robot is rewarded positively when it is in the green region $\mathcal{G}$ (behind the obstacle) shown in Fig.\ref{fig:ObstacleLocations}a and penalized when it is in the red region $\mathcal{R}$ (along the obstacle's heading direction). This reinforces \textit{spatially aware} velocities when handling dynamic obstacles i.e., velocities which move the robot away from an obstacle's heading direction, thereby reducing the risk of collision. 

\begin{prop}
Region $\mathcal{R}$ has a high risk of collision.
\end{prop}
\begin{proof}
The distance between the obstacle 
\begin{equation}
    D^2 = (p_{x}^{rob} - p_{x}^{obs})^2 + (p_{y}^{rob} - p_{y}^{obs})^2 
\label{eqn:distance}
\end{equation}

We prove that the danger of collision in the red zone is high since the distance between the dynamic obstacle and the robot is decreasing. To see this, we differentiate the equation \ref{eqn:distance} on both sides,
\begin{multline}
    2D\frac{dD}{dt} = 2(p_{x}^{rob} - p_{x}^{obs})(\frac{dp_{x}^{rob}}{dt} - \frac{dp_{x}^{obs}}{dt} ) + \\ 2(p_{y}^{rob} - p_{y}^{obs})(\frac{dp_{y}^{rob}}{dt} - \frac{dp_{y}^{obs}}{dt} ) 
\label{eqn:derivativedistance}
\end{multline} 

From the Fig.\ref{fig:ObstacleLocations}a, we get the following conditions for the case where the obstacle moves to the left (with a positive Y component in velocity) in the odometry coordinate frame. Note that the conditions also hold if the obstacle had a velocity that moved it into $\mathcal{R}$. 

\begin{multline}
if (p_{x}^{rob}, p_{y}^{rob}) \in \mathcal{R} \\
(p_{y}^{rob} - p_{y}^{obs}) > 0  \hspace{.6cm} |  \hspace{.6cm} (p_{x}^{rob} - p_{x}^{rob}) < 0 \\
(V_{y}^{rob} - V_{y}^{obs}) < 0  \hspace{.6cm} |  \hspace{.6cm} (V_{x}^{rob} - V_{x}^{rob}) > 0
\label{eqn:redzone}
\end{multline}

\begin{multline}
if (p_{x}^{rob}, p_{y}^{rob}) \in \mathcal{G} \\
(p_{y}^{rob} - p_{y}^{obs}) < 0  \hspace{.6cm} |  \hspace{.6cm} (p_{x}^{rob} - p_{x}^{rob}) < 0 \\
(V_{y}^{rob} - V_{y}^{obs}) < 0  \hspace{.6cm} |  \hspace{.6cm} (V_{x}^{rob} - V_{x}^{rob}) > 0
\label{eqn:greenzone}
\end{multline}
Equation \ref{eqn:derivativedistance} implies,
\begin{multline}
    \frac{dD}{dt} = \frac{1}{D}[(p_{x}^{rob} - p_{x}^{obs})(v_{x}^{rob} - v_{x}^{obs} ) +\\ (p_{y}^{rob} - p_{y}^{obs})(v_{y}^{rob} - v_{y}^{obs} ) ]
\label{eqn:derivativedistance2}
\end{multline} 
Substituting conditions in equation \ref{eqn:redzone} and considering comparable velocities for the robot and obstacle,
\begin{equation}
    \frac{dD}{dt} < 0
\end{equation}
So, $dist(\textbf{p}^t_{rob}, \textbf{p}^t_{obs})$ is always a decreasing function in $\mathcal{R}$. This implies a higher risk of collision. 

Substituting conditions in equation \ref{eqn:greenzone},
\begin{equation}
    \frac{dD}{dt} > 0 \hspace{.3cm} 
\end{equation} 

In $\mathcal{G}$, if we have $|(v_{y}^{rob} - v_{y}^{obs})| >> |(v_{x}^{rob} - v_{x}^{obs} )|$, then based on the signs of these components in the right hand side of equation \ref{eqn:derivativedistance} will be positive. This implies that $dist(\textbf{p}^t_{rob}, \textbf{p}^t_{obs})$ will be an increasing function in $\mathcal{G}$ if $v_{y}^{rob}$ is highly negative in y-axis. This is intuitive as a high velocity towards the negative y direction ($\mathcal{G}$ zone) is required to generate a spatially aware trajectory in the given scenario. Indirectly, velocities with highly negative $v_{y}^{rob}$ are positively rewarded in our formulation.

When the obstacle moves towards right relative to the odometry coordinate frame, the proof is symmetrical and still proves that $dist(\textbf{p}^t_{rob}, \textbf{p}^t_{obs})$ is a decreasing function in corresponding $\mathcal{R}$ constructed.
\end{proof}
 
In the case of an obstacle moving head-on the total steering reward is zero. In the presence of multiple dynamic obstacles around the robot, the union of the red zones is to be constructed for the total negative rewards. 

This is also supplemented by providing negative rewards \textit{inversely proportional to the distance} from all the obstacles in the sensor range of the robot. This reduces the danger of collision as negative reward is accumulated as the robot approaches the obstacle.  

\vspace{-15pt}
\begin{equation}
    (r_{dangerOfCollison})^t = - \cfrac{r_{dCollision}}{d_t}
\end{equation} 
 
We set $r_{goal}=2000, r_{collison}=$-$2000$, $r_{proximity }=10, r_{spatial}=25, r_{dCollison}=30$.

\vspace{3pt}

\subsubsection{Network Architecture}
The policy network architecture that we use is shown in Fig.\ref{fig:Network Architecture}.  Five 2-D convolutional layers, followed by 3 fully-connected layers are used for processing the observation space. ReLU activation is applied between the hidden layers. This architecture is much simpler and requires fewer layers for handling our observation space. 

\subsubsection{Policy Training}
We simulate multiple Turtlebot2 robots each with an attached lidar to train the models. The Turtlebots are deployed in different scenarios in the same simulation environment , to ensure that the model does not overfit to any one scenario. Our policy finishes training in less than 20 hours, which is significantly less than the 6 days it takes to train methods such as
\cite{crowdsteer,densecavoid}, which use similar training environments.

\subsubsection{Run-time Architecture}
The output of a fully trained policy network is the index \textit{i} that corresponds to a velocity pair in the action space. The $[v, \omega]$ vector at the $\textit{i}^{th}$ location in the action space is then used by the robot for navigation at the current time instant $t_c$. 
\begin{prop}
The velocity chosen by our fully trained policy will always obey the dynamics constraints of the robot.
\end{prop}
\begin{proof}
The proof follows trivially from the fact that our action space is a subset of our observation space (Fig.\ref{ObservationSpaceGeneration}c), which in turn is constructed using the dynamic feasibility equations of DWA. Thus, our policy preserves the dynamic feasibility guarantees of DWA.  
\end{proof}

\no Our full run-time architecture is shown in Fig.\ref{flow_chart}.

\section{Results, Comparisons and Evaluations}

\subsection{Implementation}
We use ROS Melodic and Gazebo 9 to create the simulation environments for training and evaluating on a workstation with an Intel Xeon 3.6GHz processor and an Nvidia GeForce RTX 2080TiGPU. We implement the policy network using TensorFlow and use the PPO2 implementation provided by stable baselines to train our policy. 

To test the policy's sim-to-real transfer and generalization capabilities, we use it to navigate a Turtlebot 2 and a Jackal robot in challenging indoor scenes with randomly moving pedestrians (see attached video). DWA-RL does not require accurate sensing of the obstacles' positions in real-world scenes.

\subsection{Training Scenario}
The training environment used to train the DWA-RL policy is shown in Fig.\ref{fig:trainingEnv}. We use 4 robots in the environment that collect the training data in parallel, speeding up the overall training process. Each robot in the training environment encounters different type of static and dynamic obstacles while navigating towards the goal, this training methodology ensures that the policy does not overfit to a particular scenario and generalizes well during the testing phase.

\begin{figure}[t]
      \centering
      \includegraphics[width=8cm, height=4cm]{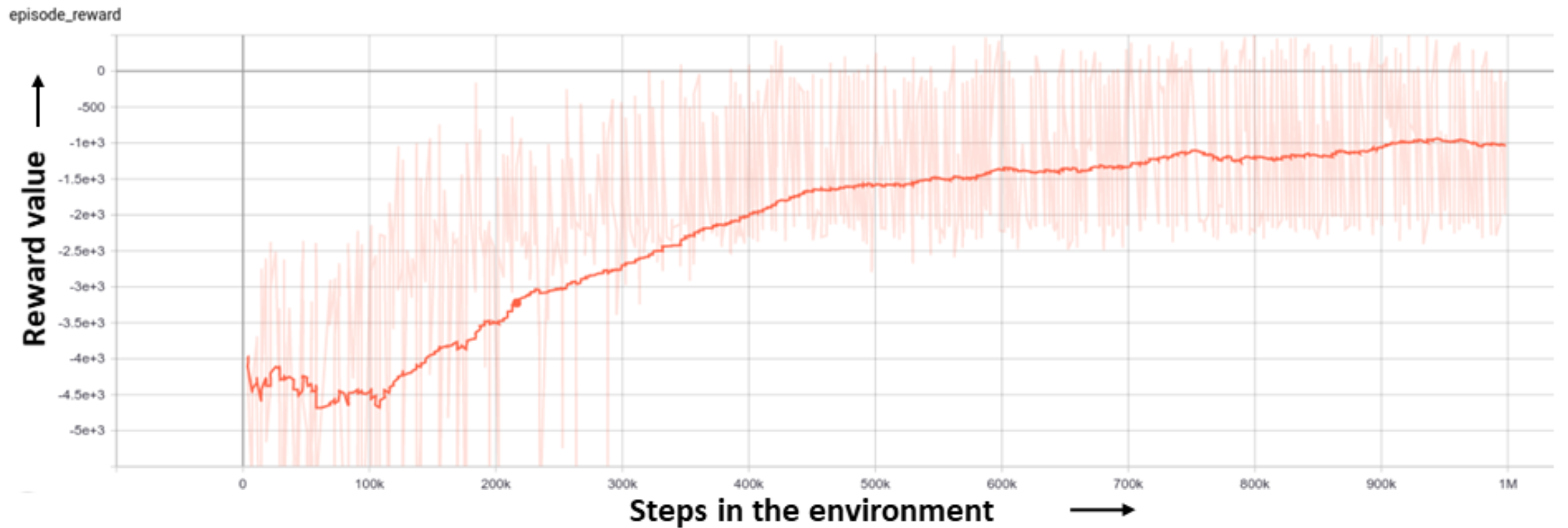}
      \caption {\small{The plot shows the reward values obtained by the policy during training. The horizontal axis shows the number of steps executed by the policy in the training environment and the vertical axis shows the reward value for each step. The policy converges to a stable reward value after executing about 700k steps in the training environment.}}
      \label{fig:convergenceGraph}
\end{figure}

\begin{figure}[t]
      \centering
      \includegraphics[width=8cm, height=4cm]{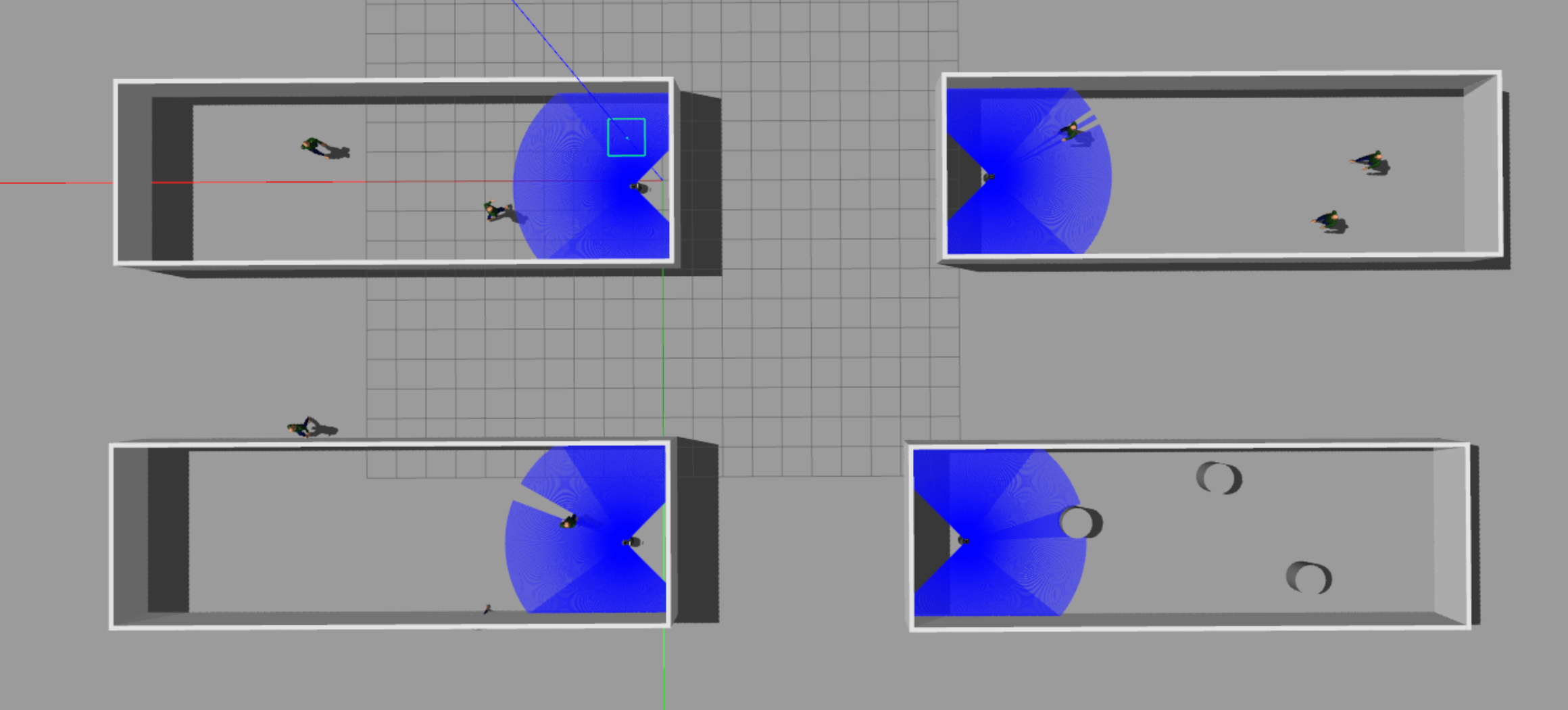}
      \caption {The training environment used to train the DWA-RL policy. We use four robots in parallel to collect the data for policy training, each facing different static and mobile obstacles.}
      \label{fig:trainingEnv}
\end{figure}

\begin{figure}[t]
      \centering
      \includegraphics[width=8cm, height=4cm]{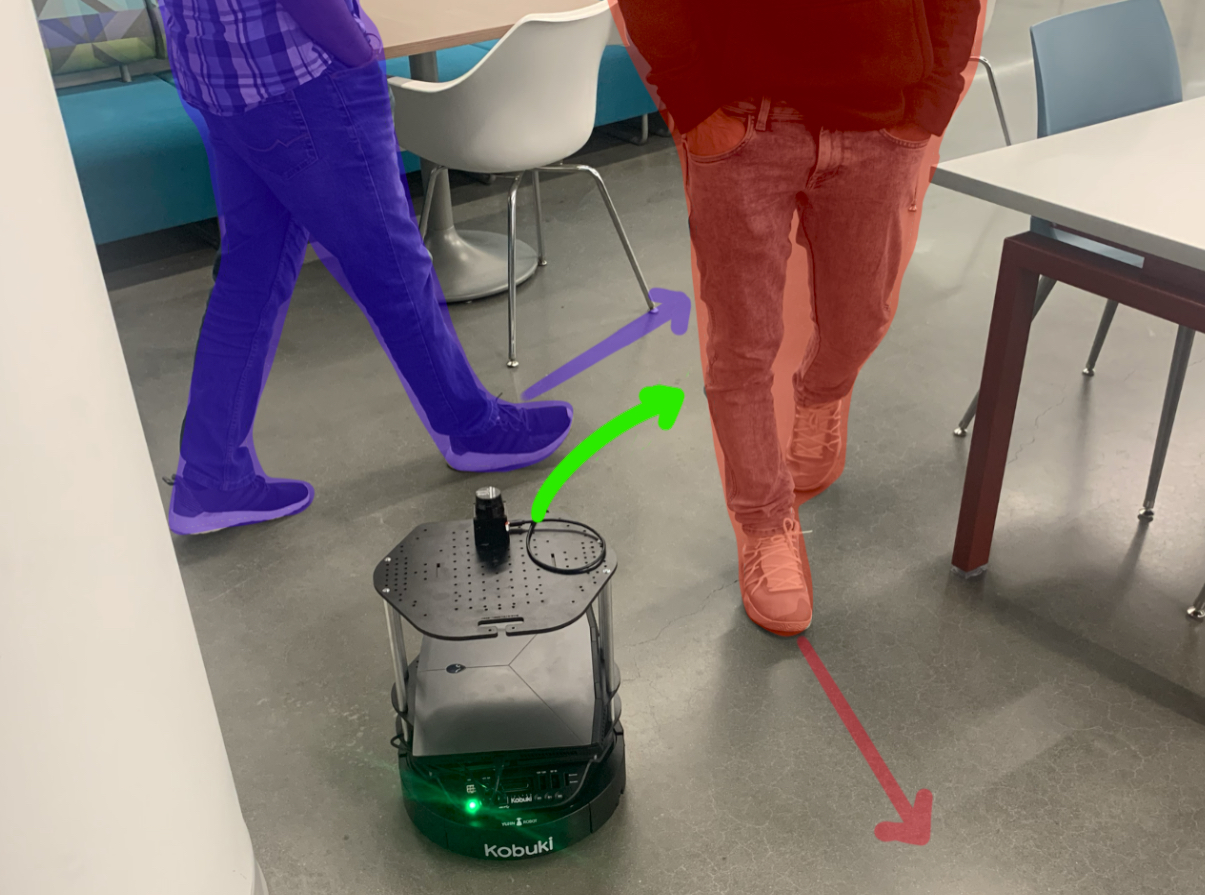}
      \caption {\small{DWA-RL tested on a Turtlebot in a dense scenario. This shows that DWA-RL can be easily transferred from simulations to different real-world robot platforms. The trajectories of the obstacles is shown in blue and red. The robot's trajectory is shown in green. }}
      \label{fig:turtlebot}
\end{figure}

\begin{figure*}[t]
\includegraphics[height=1.1in, width=\linewidth]{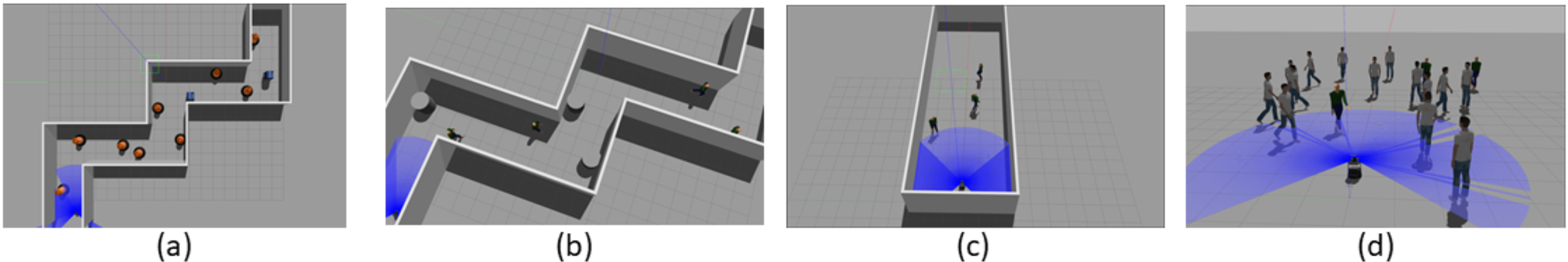}
\caption{\small{Different testing scenarios used to evaluate the collision avoidance approaches \textbf{(a)}: Zigzag-Static Scenario; \textbf{(b)}: Occluded-Ped scenario where the dynamic obstacles are suddenly introduced; \textbf{(c)}: Sparse-Dynamic scenario; \textbf{(d)} Dense-Dynamic environment contains the combination of static and dynamic obstacles.} }
\label{fig:Testing-Scenarios}
\end{figure*}


\subsection{Testing Scenarios}
 We evaluate DWA-RL and compare with prior methods in the following scenarios (see Fig.\ref{fig:Testing-Scenarios}).

\textbf{Zigzag-Static:}
This scenario contains several sharp turns with a number of static obstacles to resemble a cluttered indoor environment.

\textbf{Occluded-Ped:}
This scenario contains several sharp turns and two pedestrians who could be occluded by the walls. 

\textbf{Sparse-Dynamic:}
This scenario contains 4 walking pedestrians in a corridor-like setting moving at 45$^{\circ}$ or 90$^{\circ}$ angles with the line connecting the robot's start and goal locations.  

\textbf{Dense-Dynamic}
This scenario contains 17 pedestrians in an area of $13 \times 8 m^2$ who could be static or moving and resembles dense dynamic outdoor environments. 

\subsection{Evaluation Metrics}
We compare our approach with: (i) Dynamic Window Approach~\cite{DWA} (ii) Long et al.'s method~\cite{Multirobot_collison_avoidance}. We also provide ablation studies to demonstrate the effects of our various design choices while formulating the observation space and reward function. We use the following metrics to compare the methods and the ablation study models.

\begin{itemize}
\item \textbf{Success Rate} - The number of times the robot reached its goal without colliding with obstacles over 50 trials. The obstacles' initial positions are randomly assigned in each trial. 

\item \textbf{Average Trajectory Length} - The total distance traversed by the robot, until the goal is reached, averaged over the number of successful trials.

\item \textbf{Average Velocity} - It is the trajectory length over the time taken to reach the goal in a successful trial. 
\end{itemize}

\subsection{Analysis and Comparison}
The results of our comparisons and ablation studies are shown in tables \ref{tab:comparison_table}, \ref{tab:positive_reinforcement} and \ref{tab:observation_space_split}.

\begin{table}
\resizebox{\columnwidth}{!}{%
\begin{tabular}{|c|c|c|c|c|c|} 
\hline
\textbf{Metrics} & \textbf{Method} & \multicolumn{1}{|p{1cm}|}{\centering \textbf{Zigzag} \\ \textbf{Static}} & \multicolumn{1}{|p{1cm}|}{\centering \textbf{Occluded} \\ \textbf{Ped}} & \multicolumn{1}{|p{1cm}|}{\centering \textbf{Sparse} \\ \textbf{Dynamic}} & \multicolumn{1}{|p{1cm}|}{\centering \textbf{Dense} \\ \textbf{Dynamic}}\\ [0.5ex] 
\hline
\multirow{3}{*}{\rotatebox[origin=c]{0}{\makecell{\textbf{Success}\\\textbf{Rate}}}} & DWA-RL  & 1.0 & 1.0 & 0.54 & 0.4   \\
 & DRL & 1.0 & 0.76 & 0.3342 & .31 \\
 & {DWA} & 1.0 & 0.9 & 0.42 & 0.3 \\
\hline

\multirow{3}{*}{\rotatebox[origin=c]{0}{\makecell{\textbf{Avg}\\\textbf{Traj.}\\\textbf{Length (m)}}}} & DWA-RL & 28.85 & 27.26 & 11.95 & 12.23\\
 & DRL & 26.89 & 25.63 & 12.1 & 11.57 \\
 & {DWA} & 26.62 & 25.07 & 11.99 & 12.81 \\
\hline

\multirow{3}{*}{\rotatebox[origin=c]{0}{\makecell{\textbf{Avg}\\\textbf{Velocity (m/s)}}}} & DWA-RL & 0.38 & 0.46 & 0.42  & 0.37\\
 & DRL & 0.42 & 0.51 & 0.37 & 0.503 \\
 & {DWA} & 0.435 & 0.568 & 0.6 & 0.64 \\
\hline

\end{tabular}
}
\caption{\small{Relative performance of DWA-RL versus DWA \cite{DWA} and Long et. al's method \cite{Multirobot_collison_avoidance}.}
}
\label{tab:comparison_table}
\end{table}

From table \ref{tab:comparison_table}, we observe that in terms of success rate all approaches perform well in the Zigzag-Static scenario. However, in the environments with mobile obstacles, DWA-RL collides significantly less number of times. This is because DWA-RL considers obstacles' motion over time (in the observation space) and computes velocities that avoid the region in front of the obstacle (reinforced in reward function). DWA and Long et al.'s method try to avoid the obstacles from in-front and collide, especially in the Occluded-Ped scenario, where obstacles are introduced suddenly. Even with limited temporal information, DWA-RL always guides the robot in the direction opposite to the obstacle's motion, thereby reducing the chances of a collision. DWA-RL achieves this while maintaining a comparable average trajectory lengths and velocities for the robot.



\textbf{Ablation Study for the Positive Reinforcement:}  
We compare two policies trained with and without the positive reinforcement (PR) ($|b_t| *  r_{spatial}$) term in equation \ref{eqn:spatial-awareness} in different test environments. From Table \ref{tab:positive_reinforcement}, we observe that the policy trained with PR outperforms the model trained without it in all the test environments. The policy trained without PR mostly tries to avoid an obstacle by navigating in-front of it, predominantly resulting in collisions.    

\begin{table}
\resizebox{\columnwidth}{!}{
\begin{tabular}{|c|c|c|c|c|c|} 
\hline
\textbf{Metrics} & \textbf{Method} & \multicolumn{1}{|p{1cm}|}{\centering \textbf{Zigzag} \\ \textbf{Static}} & \multicolumn{1}{|p{1cm}|}{\centering \textbf{Occluded} \\ \textbf{Ped}} & \multicolumn{1}{|p{1cm}|}{\centering \textbf{Sparse} \\ \textbf{Dynamic}} & \multicolumn{1}{|p{1cm}|}{\centering \textbf{Dense} \\ \textbf{Dynamic}}\\ [0.5ex] 
\hline
\multirow{2}{*}{\rotatebox[origin=c]{0}{\makecell{\textbf{Success}\\\textbf{Rate}}}} & With PR  & 1.0 & 1.0 & 0.54 & 0.42  \\
& Without PR & 0.86 & 0.64 & 0.44 & 0.4  \\
\hline

\multirow{2}{*}{\rotatebox[origin=c]{0}{\makecell{\textbf{Avg Traj.}\\\textbf{Length (m)}}}} & With PR  & 28.85 & 27.26 & 11.95 & 12.46\\
 & Without PR & 28.12 & 27.86 & 11.6 & 12.78  \\
\hline

\multirow{2}{*}{\rotatebox[origin=c]{0}{\makecell{\textbf{Avg}\\\textbf{Velocity (m/s)}}}} & With PR  & 0.37 & 0.47 & 0.42 & .38 \\
 & Without PR & 0.34 & 0.41 & 0.41  & .34 \\
\hline


\end{tabular}
}
\caption{\small{Ablation study showing relative performance of the models trained with positive reinforcement (PR) and without positive reinforcement.
}}
\label{tab:positive_reinforcement}
\end{table}

\textbf{Ablation Study for the Observation Space:} Our observation space uses four matrices stacked together as show in Fig. \ref{flow_chart} which include velocities and the obstacle and goal-alignment costs. We compare this formulation with one which uses three matrices; the linear and angular velocity matrices and a total cost matrix stacked together. The total cost matrix is the sum of the obstacle and goal-alignment cost matrices. The results for both the policies are shown in Table \ref{tab:observation_space_split}. We observe that the 4-matrix formulation outperforms the 3-matrix formulation in all the scenarios. This is because, the information about environmental obstacles is better imparted into the policy when the obstacle cost is provided separately.  

\begin{table}
\resizebox{\linewidth}{!}{
\begin{tabular}{|c|c|c|c|c|c|} 
\hline
\textbf{Metrics} & \textbf{Method} & \multicolumn{1}{|p{1cm}|}{\centering \textbf{Zigzag} \\ \textbf{Static}} & \multicolumn{1}{|p{1cm}|}{\centering \textbf{Occluded} \\ \textbf{Ped}} & \multicolumn{1}{|p{1cm}|}{\centering \textbf{Sparse} \\ \textbf{Dynamic}} & \multicolumn{1}{|p{1cm}|}{\centering \textbf{Dense} \\ \textbf{Dynamic}}\\ [0.5ex] 
\hline
\multirow{2}{*}{\rotatebox[origin=c]{0}{\makecell{\textbf{Success}\\\textbf{Rate}}}} & 3-matrix & 0.82 & 0.94 & 0.36 & 0.36   \\
& 4-matrix & 1.0 & 1.0 & 0.54 & 0.42  \\
\hline

\multirow{2}{*}{\rotatebox[origin=c]{0}{\makecell{\textbf{Avg Traj.}\\\textbf{Length (m)}}}} & 3-matrix & 27.93 & 27.04 & 11.56 & 12.34\\
 & 4-matrix & 28.85 & 27.26 & 11.95 & 12.46 \\
\hline

\multirow{2}{*}{\rotatebox[origin=c]{0}{\makecell{\textbf{Avg}\\\textbf{Velocity (m/s)}}}} & 3-matrix & 0.38 & 0.44 & 0.46 & .37\\
 & 4-matrix  & 0.37 & 0.47 & 0.42 & 0.38 \\
\hline


\end{tabular}}
\caption{\small{Ablation study showing relative performance of the models trained with 3-matrix and 4-matrix observation space formulation.
}}
\label{tab:observation_space_split}
\end{table}

\textbf{Dynamics Constraints Violation} 
The Fig. \ref{fig:LinearVelocitiesDRL} and \ref{fig:AngularVelocitiesDRL} shows the graph of linear and angular velocities generated by the Long et. al's method \cite{Multirobot_collison_avoidance} in the Dense Dynamic environment. We observe that the output angular velocities lie outside the maximum and minimum attainable angular velocities of the robot 61\% of the times, leading to oscillatory/jerky motion. DWA-RL on the other hand, produces velocities that always lie within the attainable velocity range (Fig. \ref{fig:DWA-RL-LinearVelocities} and \ref{fig:DWA-RL-AngularVelocities}). This results in considerably smoother robot trajectories.    


\begin{figure}[t]
      \centering
      \includegraphics[width=\columnwidth,height=4.25cm]{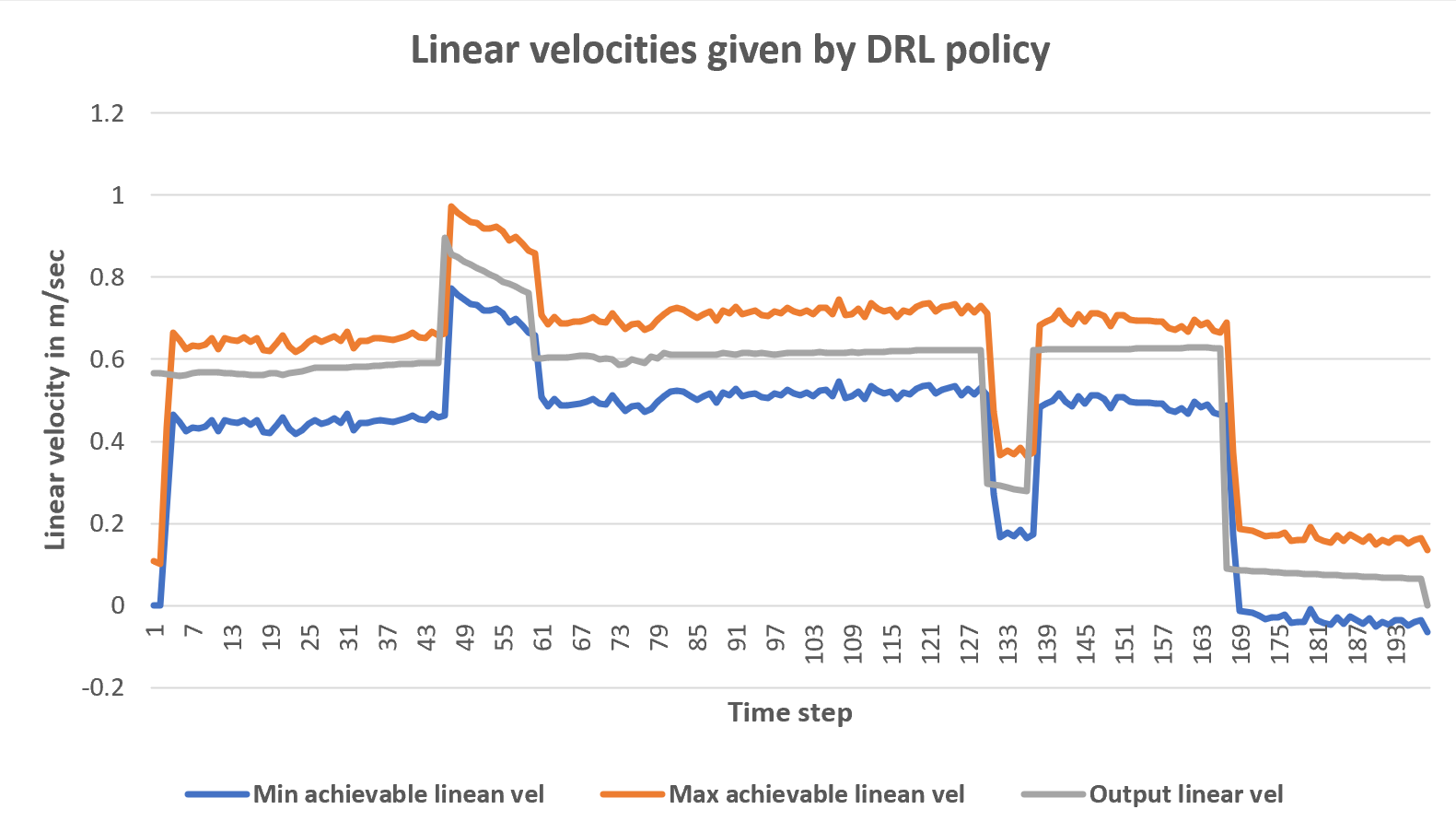}
      \caption {\small{Graph showing the change in the linear velocity generated by Long et. al's approach along with the maximum and the minimum achievable velocity at that time instant. For this experiment, we use Turtlebot 2 with max angular velocity, min angular velocity and max angular acceleration limit of 3.14 rad/s, -3.14 rad/sec and 2 $rad/s^2$ respectively.}}
      \label{fig:LinearVelocitiesDRL}
\end{figure}

\begin{figure}[t]
      \centering
      \includegraphics[width=\columnwidth,height=4.25cm]{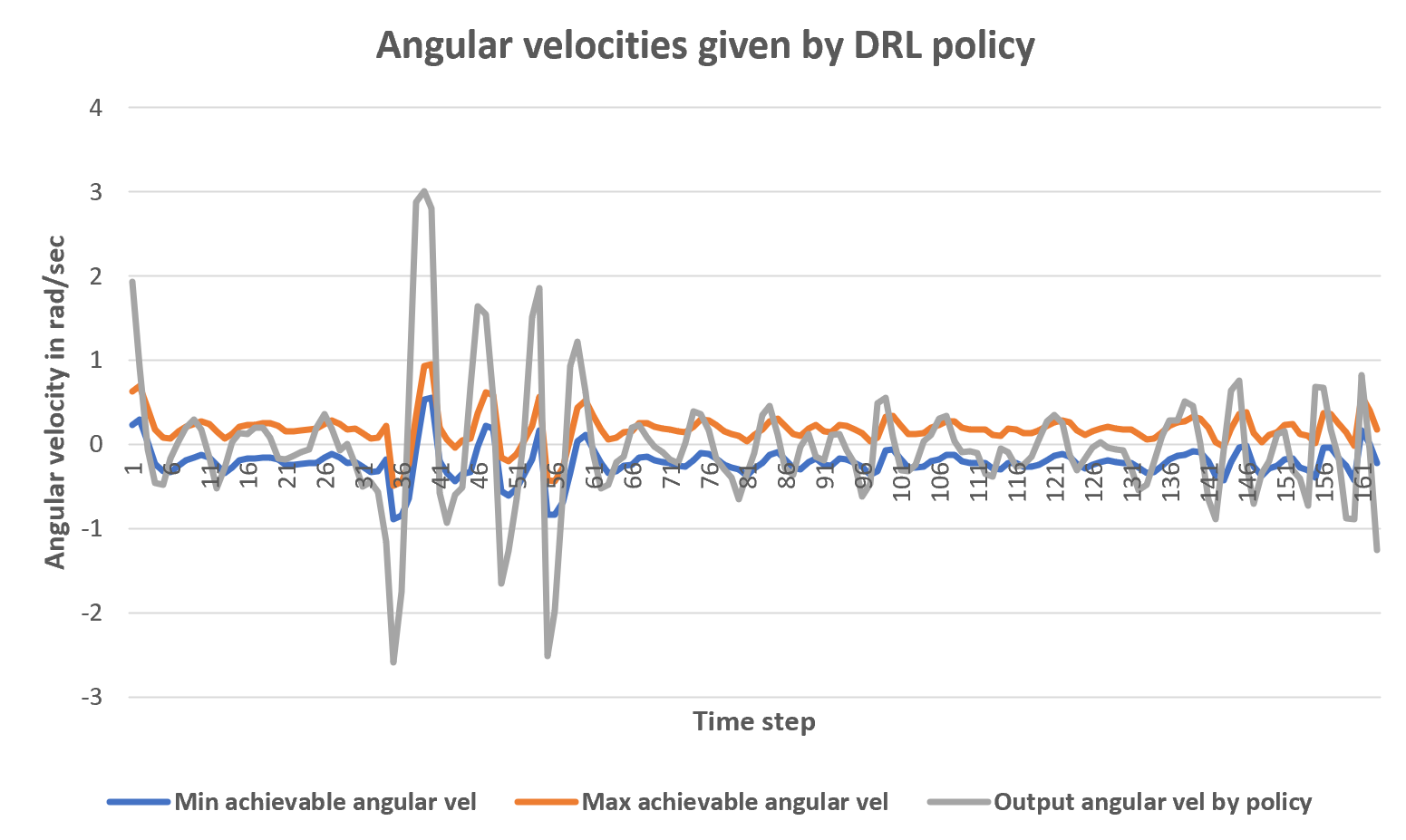}
      \caption {\small{Graph showing the change in the angular velocity generated by the Long et. al's approach along with the maximum and the minimum achievable velocity at that time instant.}}
      \label{fig:AngularVelocitiesDRL}
\end{figure}

\begin{figure}[t]
      \centering
      \includegraphics[width=\columnwidth,height=4.25cm]{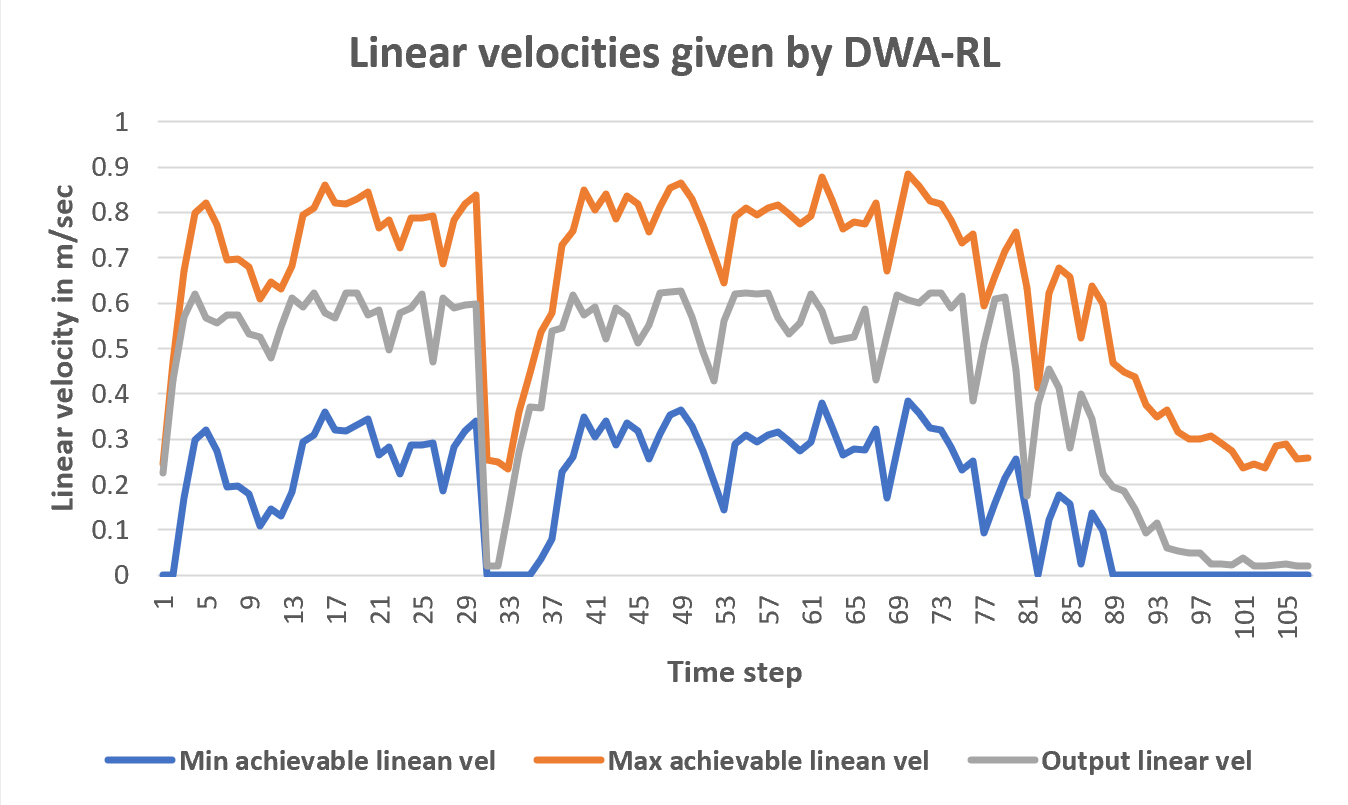}
      \caption {\small{Graph showing the change in the linear velocity generated by the DWA-RL approach along with the maximum and the minimum achievable velocity at that time instant. The plot shows that the output velocity of the DWA-RL policy is always within the achievable velocity range at any time instant.}}
      \label{fig:DWA-RL-LinearVelocities}
\end{figure}

\begin{figure}[t]
      \centering
      \includegraphics[width=\columnwidth,height=4.25cm]{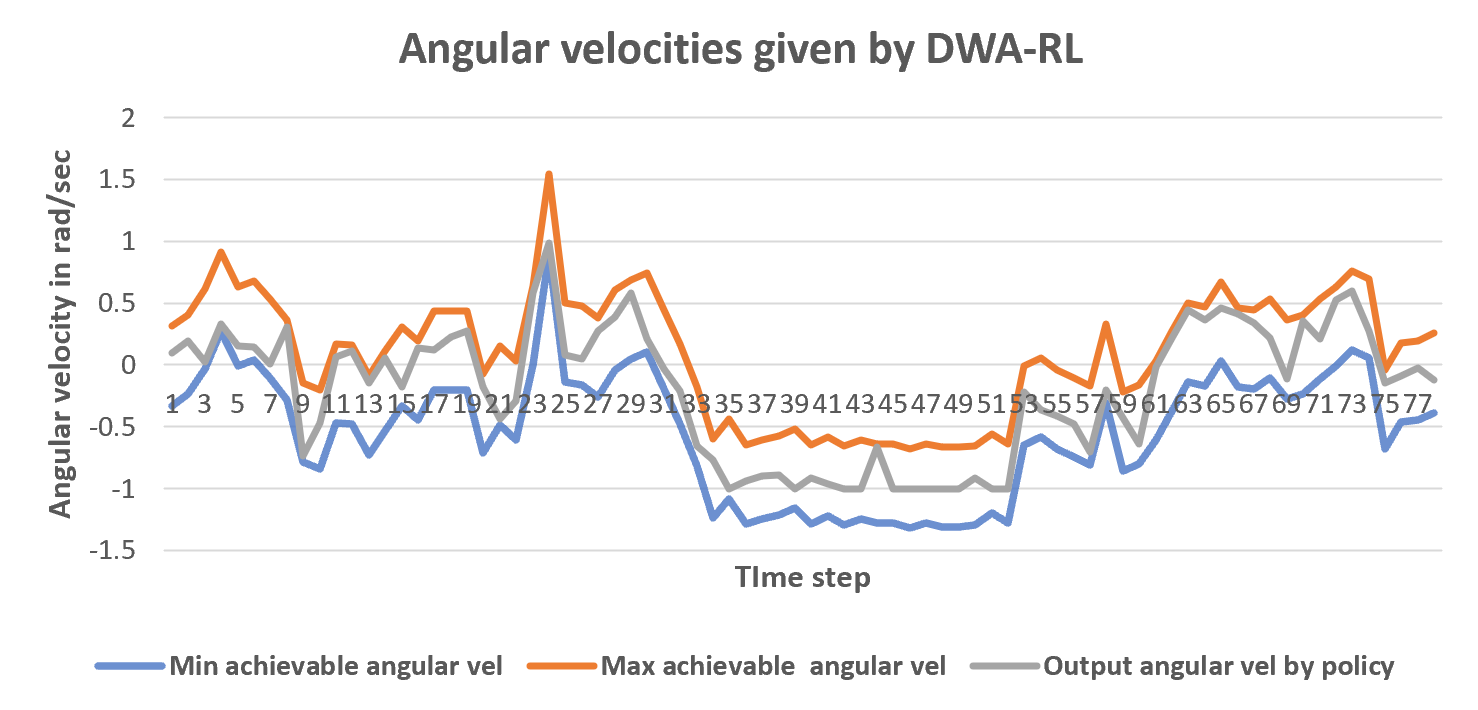}
      \caption {\small{Graph showing the change in the angular velocity generated by the DWA-RL approach along with the maximum and the minimum achievable velocity at that time instant. The plot shows that the output velocity of the DWA-RL policy is always within the achievable velocity range at any time instant.}}
      \label{fig:DWA-RL-AngularVelocities}
\end{figure}

\section{Conclusions, Limitations and Future Work}
We present a novel formulation of a Deep Reinforcement Learning policy that generates dynamically feasible and spatially aware smooth velocities. Our method addresses the issues associated with learning-based approaches (dynamic infeasible velocities) and the classical Dynamic Window Approach (sub-optimal mobile obstacle avoidance). We validate our approach in simulation and on real-world robots, and compare it with the other collision avoidance techniques in terms of collision rate, average trajectory length and velocity, and dynamics constraints violations.  

Our work has a few limitations which we wish to address in the future. For instance, the model needs at least few observations to compute a velocity that is spatially aware. If the obstacles are suddenly introduced in the field of view of the robot, the robot might freeze. Efficiency of this approach with an integrated global planner is yet to be studied. Also, the current model uses Convolutional Neural Network as layers in the policy network, but the use of LSTM \cite{LSTM} could improve the processing of the temporal data from the observation space. 



\bibliographystyle{IEEEtran}
\bibliography{References}

\end{document}